\documentclass[twoside]{article} %
\usepackage[accepted]{aistats2023}  %

\usepackage{times}  %
\usepackage{helvet}  %
\usepackage{courier}  %
\usepackage[hyphens]{url}  %
\usepackage{graphicx} %
\usepackage[round]{natbib}

\usepackage{xcolor}
\usepackage{verbatim}
\usepackage[colorlinks,citecolor=cyan,linkcolor=blue]{hyperref}
\usepackage[nameinlink,capitalize]{cleveref}
\crefformat{equation}{#2 Equation~(#1) #3}
\crefname{figure}{Figure}{Figures}
\Crefname{figure}{Figure}{Figures}
\crefname{assumption}{Assumption}{Assumptions}
\crefname{table}{Table}{Tables}
\Crefname{table}{Table}{Tables}
\crefname{equation}{Equation}{Equations}
\Crefname{equation}{Equation}{Equations}
\crefname{section}{Section}{Sections}
\Crefname{section}{Section}{Sections}
\crefname{subsection}{Subsection}{Subsections}
\Crefname{subsection}{Subsection}{Subsections}

\usepackage{microtype}
\usepackage{graphicx}
\usepackage{subfigure}
\usepackage{booktabs} %
\usepackage{todonotes}
\usepackage{bbm}
\usepackage{wrapfig} 
\usepackage{todonotes}
\usepackage{microtype}
\usepackage{graphicx}
\usepackage{booktabs} %
\usepackage{wrapfig}
\usepackage{tabularx}
\usepackage{amsmath}
\usepackage{amssymb}
\usepackage{mathrsfs}
\usepackage{amsthm}
\usepackage{stmaryrd}
\usepackage{mathtools}

\usepackage{bm}
\usepackage{dblfloatfix}
\usepackage{enumitem}
\usepackage[thinc]{esdiff}
\usepackage{bigints}
\usepackage{mdwlist}
\usepackage{thmtools, thm-restate}

\declaretheorem{definition}

\usepackage{amsmath}
\usepackage{centernot}
\DeclareMathOperator*{\rank}{rank}
\DeclareMathOperator*{\spn}{span}

\usepackage[switch]{lineno}

\usepackage{tikz-cd}

\renewcommand{\ge}{\geqslant}
\renewcommand{\geq}{\geqslant}
\renewcommand{\le}{\leqslant}
\renewcommand{\leq}{\leqslant}

\newcommand{\norm}[1]{\lVert #1 \rVert}

\newcommand{\E}{\mathbb{E}}

\newcommand{\T}{\mathsf{T}}

\newcommand{\opnorm}[1]{\norm{#1}_{\mathrm{op}}}

\newcommand{\cL}{\mathcal{L}}

\newcommand{\cT}{\mathcal{T}}

\newcommand{\rN}{\mathbb{N}}

\newcommand{\rR}{\mathbb{R}}

\DeclareMathOperator*{\argmin}{arg\,min}
\DeclareMathOperator*{\expect}{{\huge \mathbb{E}}}

\newcommand{\sspace}{\mathcal{S}}

\usepackage{algorithm}
\usepackage[noend]{algorithmic}

\begin{document}
\runningauthor{Charline Le Lan, Joshua Greaves, Jesse Farebrother, Mark Rowland, Fabian Pedregosa, Rishabh Agarwal, Marc Bellemare}
\runningtitle{A Novel Stochastic Gradient Descent Algorithm for Learning Principal Subspaces}

\twocolumn[

\aistatstitle{A Novel Stochastic Gradient Descent Algorithm \\for Learning Principal Subspaces}

\aistatsauthor{ Charline Le Lan \And Joshua Greaves \And  Jesse Farebrother \And Mark Rowland}

\aistatsaddress{ University of Oxford \And  Google Brain \And McGill University  \And DeepMind}

\aistatsauthor{Fabian Pedregosa \And Rishabh Agarwal \And Marc Bellemare}
\aistatsaddress{Google Brain \And Google Brain \And Google Brain} ]

\begin{abstract}
Many machine learning problems encode their data as a matrix with a possibly very large number of rows and columns.
In several applications like neuroscience, image compression or deep reinforcement learning, the principal subspace of such a matrix provides a useful, low-dimensional representation of individual data.
Here, we are interested in determining the $d$-dimensional principal subspace of a given matrix from sample entries, i.e. from small random submatrices.
Although a number of sample-based methods exist for this problem (e.g. Oja's rule \citep{oja1982simplified}), these assume access to full columns of the matrix or particular matrix structure such as symmetry and cannot be combined as-is with neural networks \citep{baldi1989neural}.
In this paper, we derive an algorithm that learns a principal subspace from sample entries, can be applied when the approximate subspace is represented by a neural network, and hence can be scaled to datasets with an effectively infinite number of rows and columns.
Our method consists in defining a loss function whose minimizer is the desired principal subspace, and constructing a gradient estimate of this loss whose bias can be controlled.
We complement our theoretical analysis with a series of experiments on synthetic matrices, the MNIST dataset \citep{lecun2010mnist} and the reinforcement learning domain PuddleWorld \citep{sutton1995generalization} demonstrating the usefulness of our approach.
\end{abstract}

\section{INTRODUCTION}
Learning compact representations of data while minimizing information loss is at the heart of machine learning. 
A common approach for doing so is to learn a $d$-dimensional principal subspace that explains most of the variation in the data, what is known as principal component analysis (PCA).
For small datasets, PCA can be accomplished by computing the singular value decomposition of the relevant data matrix.
For sufficiently large datasets, however, this approach is impractical and one must instead turn to a stochastic or sample-based procedure.

Streaming PCA algorithms learn an approximate principal subspace by sampling columns from the data matrix $\Psi$ and performing an incremental update that moves their approximation closer to the true subspace \citep[e.g.][]{krasulina1970method,oja1982simplified,gemp2020eigengame,gemp2021eigengame}. Central to these methods is the computation of the inner product between a full matrix column and the approximate subspace as well as a step to normalize the basis vectors parametrizing this subspace, making these methods most suited to problems where there are relatively few matrix rows.
Another line of work learns the principal subspace as the by-product of a low-rank linear regression problem. In this case, the learner forms a product $\Phi w_t$ where $\Phi$ encodes the approximate subspace and $w_t$ is a per-column weight vector; the aim is to minimize the Euclidean distance between $\Phi w_t$ and the column $\Psi_t$ \citep{srebro2003weighted, jin2016provable, sun2016guaranteed}. This approach has been effective for learning state representations in reinforcement learning \citep{bellemare2019geometric, gelada2019deepmdp,dabney2021value,lyle2021effect}, but can only handle a small number of columns, owing to the need to store an explicit weight vector for each.

In this paper, we consider the problem of learning a $d$-dimensional principal subspace by means of a neural network. Following common usage, we view the neural network as a mapping from the original input space to a $d$-dimensional vector space. We propose a fully sample-based algorithm which exhibits the best of the two classes of approaches above. Rather than maintain the weight vector $w_t$ in memory, we instead estimate it on-the-fly from samples -- effectively making the weight vector implicit.
We use the weight vector estimate to construct a gradient of a suitable loss function, on which we perform stochastic gradient descent in order to determine an approximation to the $d$-dimensional principal subspace.
Key to our approach is the derivation of the gradient in terms of Danskin's theorem.
Although the naive plug-in gradient fails to be an unbiased estimate and can perform quite poorly in practice, 
an unbiased estimate is obtained by constructing two independent weight vector estimates. These estimates are derived from a technique known as the LISSA (Linear (time) Stochastic Second-Order Algorithm, see \citet{agarwal2016second}) that produces a sequence of asymptotically-unbiased estimators of the inverse covariance matrix $(\Phi^\top \Phi)^\dagger$.
Based on its origins, we call the result the \emph{Danskin-LISSA} algorithm.

In \cref{sec:experiments}, we show that our algorithm can recover the principal subspace of synthetic matrices and of MNIST images, while only observing a small subset of the data matrix at each update.
We further demonstrate the effectiveness of our method for representation learning in reinforcement learning, specifically by learning a neural network-based approximation to the principal subspace of the successor measure \citep{blier2021learning} in the Puddle World domain \citep{sutton1995generalization}.

\section{BACKGROUND}
\label{sec:background}
\subsection{Problem Statement}
We consider a collection of {column functions} $\{ \psi_t \in \rR^S\}_{t \in \cT}$ where $\cT$ is an index set, and where each $\psi_t$ maps row indices to real values. We assume that the column indices and the row indices are drawn i.i.d from a distribution $\lambda$ on $\cT$ and $\xi$ on $\sspace$ respectively \footnote{We assume that $\xi(s)>0$ for all row indices $s \in \sspace$ and that $\lambda(t)>0$ for all column indices $t \in \cT$.}. For a given integer $d \in \rN$ and a \emph{row representation}
$\phi : \sspace \to \rR^d$, we define the \emph{representation loss}
\begin{equation}\label{eqn:representation_loss}
    \cL(\phi) = \expect_{t \sim \lambda} \Big [ \min_{w_t \in \rR^d} \expect_{s \sim \xi} \big [ ( \phi(s)^\top w_t - \psi_t(s) )^2 \big ] \Big ] .
\end{equation}
The representation loss describes the approximation error incurred by fitting the column function $\psi_t$ with the $d$-dimensional linear approximation $\phi(s)^\top w_t$, on average over draws from $\lambda$.
Here, we are interested in determining a $d$-dimensional representation $\phi$ that minimises $\cL(\phi)$ among all such representations.

For now, let us consider the case in which $\sspace$ and $\cT$ are of finite sizes $S$ and $T$, respectively. In this case, we may write $\Phi \in \rR^{S \times d}$ for the \textit{feature matrix} whose rows are $\big(\phi(s)\big)_{s \in \sspace}$ and $\Psi \in \rR^{S \times T}$ for the data matrix whose columns are $\big(\psi_t\big)_{t \in \cT}$. If additionally $W \in \rR^{d \times T}$ is a weight matrix, then finding the function $\phi$ that minimizes Equation \ref{eqn:representation_loss} is equivalent to jointly minimizing the loss $\cL(\Phi, W)$ over $\Phi$ and $W$, where
\begin{equation}\label{eqn:low_rank_approximation}
    \cL(\Phi, W)= \|\Xi^{1/2}(\Phi W -\Psi) \Lambda^{1/2} \|^2_F.
\end{equation}
Here, $\Xi \in \rR^{S \times S}$ (resp. $\Lambda \in \rR^{T \times T}$) is a diagonal matrix with entries $\{\xi(s): s \in \sspace \}$ (resp. $\{ \lambda(t): t \in \cT \}$) on the diagonal.
For a given $\Phi$, we write
\begin{equation}\label{eq:def_w}
    W^*_\Phi \in \argmin_{W \in \rR^{d \times T}} \cL(\Phi, W) \qquad \cL(\phi) = \cL(\Phi, W^*_\Phi) .
\end{equation}
From standard linear algebra (see \cref{lemma:w_star} in \cref{app:proofs}), in closed form we have
\begin{equation}
    W^*_\Phi = (\Phi^\top \Xi \Phi)^\dagger \Phi^\top \Xi \Psi . \label{eqn:optimal_weight_vector}
\end{equation}
Note that this expression does not depend on the column distribution $\Lambda$. We will use this matrix form to derive a gradient-based algorithm in the next section.

Equation \ref{eqn:low_rank_approximation} describes a weighted low-rank approximation problem \citep{srebro2003weighted}. Its solutions are the set of matrices $\Phi$ whose columns span the $d$-dimensional subspace of left singular vectors of $\Psi$ with respect to the inner product $(x, y)_\Xi = x^\T \Xi y$ (see \cref{prop:svd} in \cref{app:proofs} for a proof). If in addition the columns of $\Psi$ have mean zero, this corresponds to determining the subspace spanned by the $d$ principal components of $\Psi$.
Consequently, in the finite case our objective is to find a state representation whose implied feature matrix has columns that span this subspace. As we will see, one advantage of this objective over the more usual Rayleigh quotient in the case $d=1$,
\begin{equation*}
    \expect_{s, s' \sim \xi, t \sim \lambda} \big [ \phi(s) \psi_t(s) \psi_t(s') \phi(s') ] ,
\end{equation*}
is that its gradient incorporates an error term $\mathbb{E}_{s \sim \xi, t \sim \lambda}[(\phi(s)^\top w_t - \psi_t(s)) w_t^\top]$ that is naturally zero at a minimizer.

\section{PCA FROM SAMPLES}
\label{sec:method}
\label{sec:algo}
We assume access to a model from which we may repeatedly sample row indices according to the distribution $\xi$ and the values taken on at those row indices by column functions sampled from $\lambda$.
We are interested in the setting in which it is undesirable or impossible to sample the entire collection of column functions for a given state, or an entire column function all at once. This is different from the setting that approaches such as Oja's method \citep{oja1982simplified} or the recent EigenGame \citep{gemp2020eigengame} have considered for their experiments,
which in matrix terms assume that it is possible to sample entire rows or columns from $\Psi$ (for a longer discussion on prior work, see \cref{sec:relatedwork}).

Let us begin by expressing the gradient of the loss function $\cL(\Phi, W)$. In matrix form, this is
\begin{align}\label{eq:formula_gradient}
    \nabla_\Phi \cL(\Phi, W) &= 2 \Xi (\Phi W - \Psi) \Lambda W^\top\\
      \nabla_W \cL(\Phi, W) &= 2 \Phi^\top \Xi (\Phi W - \Psi) \Lambda
\end{align}
When the number of columns $T$ is small, finding an optimal $\phi$ can be accomplished by optimizing the loss function $\cL(\Phi, W)$ using a nested or two-timescale optimization procedure based on unbiased estimates of these gradients. For example, the pair of update rules 
\begin{align}
    \phi(s) &\gets \phi(s) - \alpha (\phi(s)^\top w_t - \psi_i(s)) w_t \nonumber \\
    w_t &\gets w_t - \beta  \phi(s) (\phi(s)^\top w_t - \psi_i(s)) \label{eqn:explicit_method}
\end{align}
finds an optimal representation $\phi$ under suitable conditions on the step-sizes $\alpha$ and $\beta$.
This is because the loss $\cL(\Phi, W)$ is convex in $W$ when $\Phi$ is fixed and the two-timescale algorithm allows us to approximately run gradient descent on the objective we care about.

When $T$ is large (or infinite), however, it may be expensive (or impossible) to store a separate weight vector for each column.
Instead, we rely on a form of the gradient of the loss $\cL(\phi)$ in which the weight vector is implicit.
\begin{restatable}{lemma}{danskin}
   Let $\beta >0$ be a regularization parameter. The loss $\cL: \rR^{S \times d} \rightarrow \rR$ defined by
    \begin{align}\label{eq:define-phi-loss}
        \cL(\Phi) = \min_{W \in \mathbb{R}^{d \times T}}\left( \| \Xi^{1/2} (\Phi W -\Psi ) \Lambda^{1/2} \|^2_F + \beta \| W \|_F^2 \right)
    \end{align}
    is continuously differentiable, with gradient
    \begin{align*}
        \nabla_\Phi \cL(\Phi) = 2 \Xi(\Phi W^*_\Phi - \Psi) \Lambda {W_\Phi^*}^\top \, ,
    \end{align*}
    where
    \begin{align}\label{eq:w-opt}
        W_\Phi^* = (\Phi^\top \Xi \Phi + \beta I)^{-1} \Phi^\top \Xi \Psi \, .
    \end{align}
    \label{lemma:danskin}
\end{restatable}
\begin{proof}
The proof is similar to the one of Danskin's theorem \citep{danskin2012theory}.
By linear algebra, the unique minimizer $W_\Phi^*$ in Equation~\eqref{eq:define-phi-loss} is given by Equation~\eqref{eq:w-opt}, which is itself differentiable with respect to $\Phi$.
By the chain rule, we have
\begin{equation}
     \nabla_\Phi \cL(\Phi)  = \nabla_\Phi \cL(\Phi, W^*_{\Phi}) + \left(\frac{\partial W^*_{\Phi}}{\partial \Phi}\,\right)^\top\frac{\partial}{\partial W^*_{\Phi}} \cL(\Phi, W^*_{\Phi}).
\end{equation}
Now, since $W^*_{\Phi}$ is defined as the (unconstrained) minimizer of $\cL(\Phi, W^*_{\Phi})$, its gradient with respect to the second argument vanishes at $W^*_{\Phi}$, and so second term is zero. The result then follows from the definition of $\nabla_\Phi \cL(\Phi, W)$ in Equation~\eqref{eq:formula_gradient}. \qedhere
\end{proof}

The idea is to use an instantaneous estimate of $W^*_\Phi$ to update the row representation in the negative direction of the (estimated) gradient of $\cL(\phi)$.
As we will see, such an estimate can be obtained by sampling as little as a single column and a small number of rows. 
In effect,  given a sample row index $s$ our goal is to obtain a gradient estimate $\hat g(s)$ such that
\begin{equation}\label{eqn:update_rule_implicit}
    \phi(s) \gets \phi(s) - \alpha \hat g(s)
\end{equation}
should converge to an optimal representation under suitable conditions on the time-varying step-size $\alpha$. In \cref{sec:deep-svd}, we will discuss how Equation \ref{eqn:update_rule_implicit} can be applied to learn parametrized row representations such as those described by neural networks.

Before describing our approach, it is worth noting that the procedure that naively estimates $W^*_\Phi$ from a subset of rows and columns results in a biased gradient estimate.
That is, suppose we are given the sample row indices $s, s', s_1, \dots s_n$ and sample column $t$. If we write $\hat \Phi$ for the matrix whose rows are $\phi(s_1), \dots, \phi(s_n)$ and construct the empirical covariance matrix $\hat C = \hat \Phi^\top \hat \Phi$, then we find that the estimate
\begin{equation}
    \hat g_{\textsc{naive}}(s) = \hat w_t \big( \phi(s)^\top \hat w_t - \psi_t(s) \big) \qquad \hat w_t = \hat C^\dagger \phi(s') \psi_{t}(s') \label{eqn:naive_estimator}
\end{equation}
is not an unbiased estimate of $\nabla_{\phi(s)} \cL(\Phi)$. In fact, the bias can be quite substantial when $n$ is small, as we empirically show in \cref{sec:experiments}.

\subsection{An Improved Gradient Estimate}

One issue with the estimate of Equation \ref{eqn:naive_estimator} is that the estimated weight vector $\hat w_t$ is itself a largely biased estimate of the optimal weight vector for column $t$ (that is, the $t$\textsuperscript{th} column of $W^*_\Phi$, $W^*_{\Phi, t}$). Conversely, unbiasedness is obtained if $\hat w_t$ satisfies
\begin{equation*}
    \expect [ \hat w_t ] = W^*_{\Phi, t},
\end{equation*}
and if the term $\hat w_t^\top$ is an independent, also unbiased estimate of ${W^*_{\Phi, t}}^\top$ in \cref{lemma:danskin}.
To reduce the bias of the naive estimate, we will 
construct two low-biased estimates of the inverse covariance matrix $(\Phi^\top \Phi)^\dagger$, $\hat C$ and $\hat C'$, from which we derive two independent weight estimates $\hat w_t$ and $\hat w_t'$.

Before we explain how to obtain these estimates, let us describe our algorithm at a high level. We begin by drawing three row indices $s, s', s''$ and a column index $t$. We then construct the weight estimates
\begin{equation*}
    \hat w_t = \hat C \phi(s') \psi_t(s') \qquad \hat w_t' = \hat C' \phi(s'') \psi_t(s'') ,
\end{equation*}
and then the gradient estimate
\begin{equation}\label{eqn:gradient_estimate_2rr}
    \hat g_{\textsc{dl}}(s) = \hat w_t' \big(\phi(s)^\top \hat w_t - \psi_t(s)\big).
\end{equation}
which uses two LISSA estimators~\citep{agarwal2016second} to construct independent weight estimates by application of Danskin's theorem.
In effect,  using two separate weight estimates effectively allows us to estimate the outer product $W^*_{\Phi,t} \big(W^*_{\Phi, t}\big)^\top$ appearing in \cref{lemma:danskin} with a very low bias and hence obtain a gradient estimate that is overall low-biased,
up to a multiplicative factor that we fold into the step-size parameter.
\begin{restatable}{theorem}{unbiasedestimator}
\label{thm:the_2rr_gradient_is_unbiased}
Let $e_s \in \rR^S$ denote a basis vector. Given two independent unbiased estimates $\hat{C}$ and $\hat{C}^\prime$ of the inverse covariance,
for $s \sim \xi$, the gradient estimate $\hat g_{\textsc{DL}}(s)$ given in \cref{eqn:gradient_estimate_2rr} satisfies
\begin{equation*}
    \expect [ e_s \hat g_{\textsc{dl}}(s)^\top ] = \Xi (\Phi W^*_\Phi - \Psi) \Lambda {W^*_{\Phi}}^\top.
\end{equation*}
\end{restatable}
Note that the estimate $ \hat g_{\textsc{dl}}(s)$ does not require the set of columns $\cT$ to be finite. As such, our procedure can also be used to learn the principal components of infinite sets of columns; we will demonstrate this point in \cref{sec:deep-svd}.

\subsection{Estimate of the Weight Vector $W^*_{\Phi,t}$}
\label{sec:weightlissa}
We begin by deriving a procedure which, given access to a stream of sample row representations
$\big(\phi(s_j)\big)_{j=1}^\infty$, asymptotically produces an unbiased estimate of the optimal weight vector for a given column $t$.

Central to our procedure is an estimate $\hat C$ of the inverse covariance matrix $(\Phi^\top \Xi \Phi)^\dagger$. We construct this estimate by embedding what is known as the \emph{LISSA estimator} \citep[originally used to estimate the Hessian inverse]{agarwal2016second}.
Our algorithm is parameterised by two scalars, $\kappa$ and $J$, which trade off estimator variance with sample complexity.
All proofs can be found in \cref{app:proofsalgo}.

To begin, consider an arbitrary matrix $\Phi \in \rR^{S \times d}$ and denote $\opnorm{\cdot}$ the spectral norm.
For any $\kappa < \opnorm{\Phi^\top\Xi \Phi}^{-1}$, the Moore-Penrose pseudo-inverse of $(\Phi^\top\Xi \Phi)^\dag$ has a Neumann series expansion of the form
\begin{align}\label{eqn:neumann_series_for_pseudoinverse}
    (\Phi^\top\Xi \Phi)^\dag = \kappa \sum_{i=0}^\infty (I - \kappa \Phi^\top\Xi \Phi)^i.
\end{align}
Here, $\kappa$ is a scaling parameter that ensures the convergence of the series.
Denoting $S_j$ the first $j$ terms of the above series, we have that
\begin{align*}
    S_j =  \kappa I + (I -  \kappa \Phi^\top\Xi \Phi)S_{j-1}.
\end{align*}
We use this observation to build an estimator of $(\Phi^\top\Xi \Phi)^\dag$ with access to a finite number of samples from $\sspace$.
\begin{definition}[LISSA estimator]
Let $\Phi \in \rR^{S \times d}$ be a feature matrix.
Let $s_{1:J}=\{s_1, s_2, ..., s_J \}$ be $J$ i.i.d. row indices sampled from $\xi$. Let $\kappa_0 \in (0, 2)$ and $\kappa = \kappa_0  \sup_{s_{1:J}} \|\phi(s_i)\|^{-2}_2$. The $j$-LISSA estimator $\widehat{\Delta}_j$ is recursively given by
\begin{align}
    \widehat{\Delta}_0 &= \kappa I \nonumber \\
    \widehat{\Delta}_j &= \kappa I + (I - \kappa \phi(s_j) \phi(s_j)^\top)\widehat{\Delta}_{j-1}, \; 0 < j \le J . \label{eq:lissa-recursion}
\end{align}\label{defn:lissa_estimator}
\end{definition}
\begin{restatable}[Bias of LISSA]{lemma}{bias}
For $\kappa < \sup_{s_{1:J}} 2 \|\phi(s_i)\|^{-2}_2$, the bias of $\widehat{\Delta}_{j}$ with respect to $(\Phi^\top \Xi \Phi)^{\dag}$ is given by
\begin{align*}
 \mathbb{E}(\widehat{\Delta}_{j}) -  (\Phi^\top \Xi \Phi)^{\dag} 
  &=  -( \Phi^\top \Xi \Phi)^{\dag} (I -  \kappa \Phi^\top \Xi \Phi)^{j+1} 
\end{align*}
In particular, this bias asymptotically vanishes, in the sense that
\begin{equation*}
    \lim_{j \to \infty} \mathbb{E}(\widehat{\Delta}_{j}) -  (\Phi^\top \Xi \Phi)^{\dag}  =0 .
\end{equation*}
\label{lemma:bias}
\end{restatable}
While for any finite value of $J$, the LISSA estimator $\widehat{\Delta}_j$ is not an unbiased estimate, \cref{lemma:bias} establishes that its bias can be made arbitrarily small with enough samples. In our experiments, we will show that with few row samples this results in substantially better convergence compared to a naive estimate of the covariance matrix.

In \cref{defn:lissa_estimator}, the parameter $\kappa$ controls the rate of convergence of the full Neumann series: larger
values of $\kappa$ result in faster convergence, requiring fewer samples to obtain an estimate that has little bias with regards to the inverse covariance matrix. However, larger 
values of $\kappa$ ($\kappa$ is bounded above as per \cref{defn:lissa_estimator}) also produce estimators that have higher variance.
Although here we consider the simplest setting in which a single sample is used at each iteration $j$ in Equation \ref{eq:lissa-recursion}, %
the variance of the estimator can of course be reduced by using several samples per iteration. 

\subsection{Algorithm Based on LISSA}
Provided that we use the LISSA procedure twice to construct two independent estimates $\hat w_t$, $\hat w'_t$ of the optimal weight vector $W^*_{\Phi,t}$, it is straightforward to demonstrate that $\hat g_{\textsc{dl}}(s)$ (Equation \ref{eqn:gradient_estimate_2rr}) becomes an unbiased estimate of the gradient of the loss $\mathcal{L}(\Phi)$ as $J \to \infty$; furthermore, for finite $J$ its bias is controlled as a consequence from \cref{lemma:bias}. We may then perform gradient descent with this estimate, adjusting the $s$\textsuperscript{th} row of the matrix $\Phi$ according to
\begin{equation}\label{eqn:tabular_update_with_two_rr}
    \phi(s) \gets \phi(s) - \alpha \hat g_{\textsc{dl}}(s),
\end{equation}
where $\alpha \in [0, 1)$ is a suitable step size. Based on our derivation, we call this procedure the Danskin-LISSA algorithm. In practice, it is usually desirable to update $\phi$ for $N > 1$ rows at once and use $M > 1$ samples to estimate $\hat w_i$ and $\hat w'_i$; we give this more general form in Algorithm \ref{alg1}. Note that while larger values of $J$ are desirable in order to reduce estimation bias, larger values of $M$ and $N$ contribute to reducing the variance of the gradient estimate $\hat g_{\textsc{dl}}$ and speeding up the learning process.

An important case is when the row representation $\phi$ is given by a mapping that is parametrized by a collection of weights $\theta$, in particular a neural network. In this case, Equation \ref{eqn:tabular_update_with_two_rr} should be replaced by an update rule that adjusts the weights $\theta$. In practice, this can be done by determining the Jacobian $\frac{\partial{\phi}}{\partial {\theta}}$ of $\phi$ with respect to the weights $\theta$, and applying the update
\begin{equation*}
    \theta \gets \theta - \alpha \frac{\partial{\phi}}{\partial {\theta}} \hat g_{\textsc{dl}}(s) .
\end{equation*}
An alternative particularly suited to automatic differentiation frameworks \citep{bradbury2018jax, abadi2016tensorflow, paszke2017automatic}, is to define a loss function whose gradient corresponds to $ \frac{\partial{\phi}}{\partial {\theta}} \hat g_{\textsc{dl}}(s)$. One can verify that the sample loss function
\begin{align*}
    \tfrac{1}{2} \big(\ell(\hat w_t + \hat w_t') - \ell(\hat w_t) - \ell(\hat w_t')\big) \\
    \ell(w) = \big(\phi(s)^\top \textsc{sg}(w) - \psi_t(s) \big)^2
\end{align*}
satisfies this requirement, where $\textsc{sg}$ denotes the stop-gradient operation (in the sense that $\nabla_\theta \textsc{sg}(w) = 0$). Additionally, the recursion in Equation \ref{eq:lissa-recursion} can be implemented efficiently by first computing the vector-matrix product $\phi(s_j)^\top \widehat{\Delta}_{j-1}$ and then taking the outer product of the result with $\phi(s_j)$.

\begin{algorithm}[H]
\caption{Danskin-LISSA}\label{alg1}
\begin{algorithmic}[1]
  \STATE \textbf{Parameters}: Dimension $d \in \rN^+$, $J, M, N \in \rN^+$, $\alpha$, $\kappa_0 \in (0, 2)$
  \REPEAT
  \STATE Sample independent rows $s_{1:N}, s'_{1:M}, s''_{1:M} \sim \xi$
  \STATE Sample a column $t \sim \lambda$
  \STATE $\hat C \gets \textsc{lissa}(\kappa_0, J)$
  \STATE $\hat C' \gets \textsc{lissa}(\kappa_0, J)$
  \STATE $w_t = \hat C \sum_{k=1}^M \phi(s'_k) \psi_t(s'_k)$
  \STATE $\hat w_t' = \hat C' \sum_{k=1}^M \phi(s''_k) \psi_t(s''_k)$
  \STATE $\hat g_{\textsc{2lissa}}(s_k) = \hat w_t' \big(\phi(s_k)^\top \hat w_t - \psi_t(s_k)\big)$
  \STATE $\phi(s_k) \gets \phi(s_k) - \alpha \hat g_{\textsc{dl}}(s_k)$ for $k = 1, \dots, N$
  \UNTIL{satisfied}
\end{algorithmic}
\end{algorithm}

\section{RELATED WORK}
\label{sec:relatedwork}
\paragraph{Streaming PCA.}
\looseness=-1\citet{oja1982simplified} and \citet{krasulina1970method} proposed the original streaming PCA algorithms. They approximate the top eigenvector of a matrix through a stochastic approximation of the power method. \citet{tang2019exponentially} extends this method to other principal components but requires explicit normalization. \citet{amid2020implicit} extends it without the need to explicitly performing orthonormalization after each gradient step at the cost of a batch having to be of size $1$.

\citet{pfau2018spectral} recovers the subspace spanned by the top eigenfunctions of \emph{symmetric} infinite dimensional matrices by parametrizing them with neural networks and performing gradient descent on a kernel-based loss. It is itself a generalization of slow feature analysis \citep{wiskott2002slow} in the tabular setting. \citet{deng2022neuralef} extends the objective from \citet{gemp2020eigengame} to the function space and propose an algorithm to learn the top $d$-eigenfunctions of symmetric matrices by representing them with $d$ neural networks. To find the principal subspace of a general infinite dimensional matrix $\Psi$, the approaches above require computing eigenfunctions of $\Psi \Psi^T$, which requires full row access to $\Psi$. 
By contrast, our method can recover the principal subspace of any infinite dimensional matrix using samples entries from rows of $\Psi$. 

\paragraph{Low-rank matrix completion.}
In this setting, we observe a subset of entries from a data matrix and aim to find a low-rank matrix that matches these observations \citep{srebro2003weighted}. Matrix factorization is a common technique to solve this problem where the matrix of interest is expressed as a product $\Phi W$. It can be solved efficiently by standard optimization algorithms \citep{sun2016guaranteed}. \citet{hardt2014understanding, jain2013low} rely on alternating minimization over the representation and weight matrices and guarantee convergence towards the true matrix. Other methods perform gradient descent \citep{li2019symmetry, ye2021global} or stochastic gradient descent \citep{jin2016provable, ge2015escaping, de2015global}. \citet{keshavan2010matrix, keshavan2009gradient} minimize simultaneously over the representation $\Phi$ and the weights $W$ by gradient descent. \citet{dai2010set} first solves the inner optimization problem and find the optimal weight matrix $W$ and then takes a gradient step on the outer optimization problem, with respect to the representation matrix $\Phi$. The Grassmannian Rank-One Subspace Estimation (GROUSE) algorithm \citep{balzano2010online} is a stochastic manifold gradient descent algorithm for tracking subspaces from incomplete data which was recently shown to be equivalent to Oja's algorithm \citep{balzano2022equivalence}. In comparison, we consider the problem of learning low-dimensional embeddings of higher dimensional vectors through neural networks and propose an optimization procedure which performs gradient descent on the representation matrix $\Phi$ only and where the weight matrix $W^*_\Phi$ is expressed implicitly, as a function of $\Phi$.

\section{EXPERIMENTS}
\label{sec:experiments}
\begin{figure*}[t]
    \centering
 \includegraphics[width=0.49\textwidth]{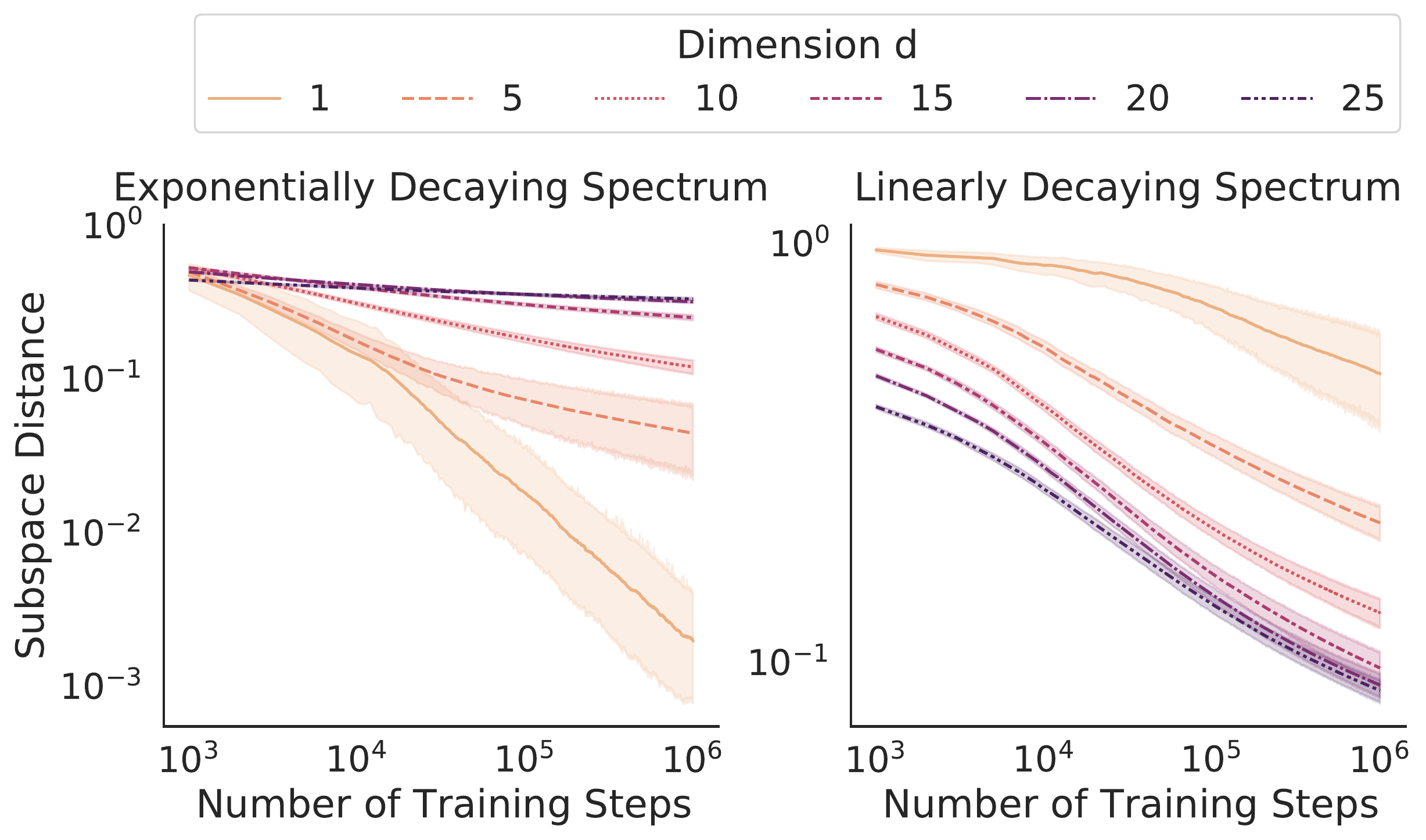}~~
 \includegraphics[width=0.49\textwidth]{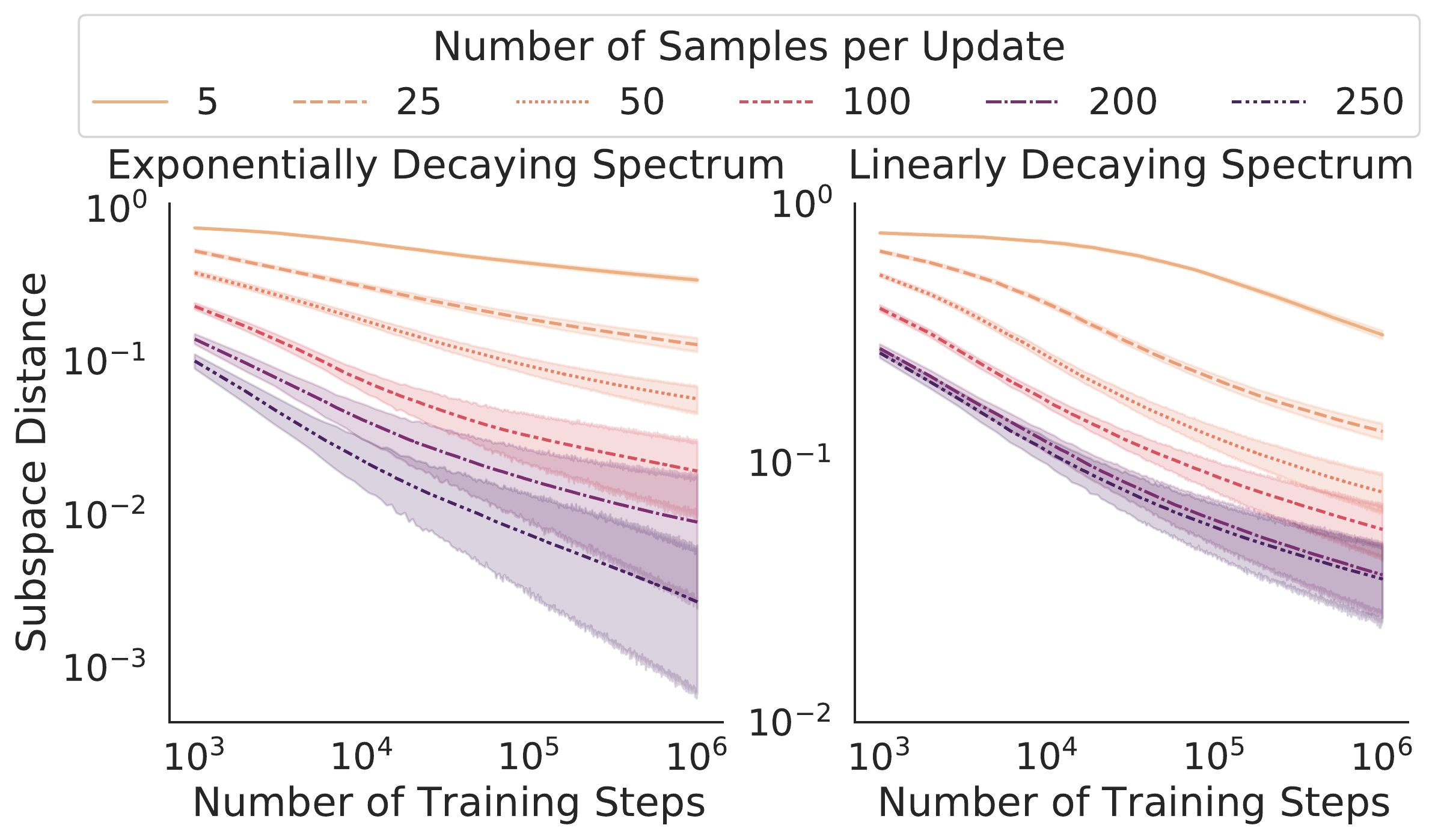}
    \caption{\looseness=-1 Subspace distance over the course of training LISSA for different dimensions (\textbf{Left}, $L=25$) and for different total number of samples per update (\textbf{Right}, $d=10$) on synthetic matrices with a spectrum decaying linearly and exponentially, averaged over 30 seeds. Shaded areas represent 95$\%$ confidence intervals.}
    \label{fig:plot1a}
\end{figure*}

\begin{figure*}[t]
    \centering
    \includegraphics[width=0.7\textwidth]{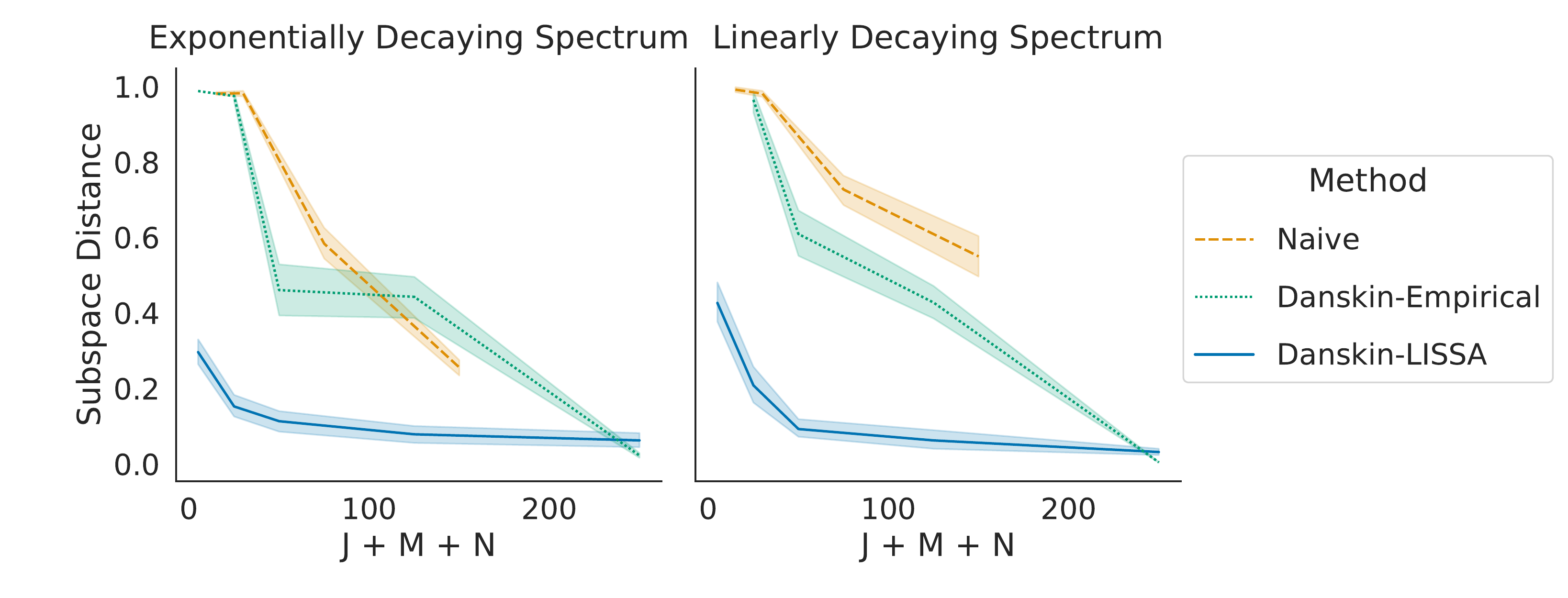}
    \vspace{-0.3cm}
    \caption{\looseness=-1 Subspace distance ($d = 10$) after $10^6$ training steps according to the method used to estimate the loss gradient. Here, the $x$ axis represents the total number of row samples $L$ from the $\Phi$ matrix ($L = 2J + 2M + N$ for the Danskin methods, $J + M + N$ for the naive method). Shaded areas represent 95$\%$ confidence intervals. Note that because we are sampling with replacement, $L = 250$ still differs from the gradient given in Lemma \ref{lemma:danskin}.}
    \label{fig:plot2}
\end{figure*}

\begin{figure*}[t]
    \centering
    \vspace{-0.2cm}
    \includegraphics[width=0.3\textwidth]{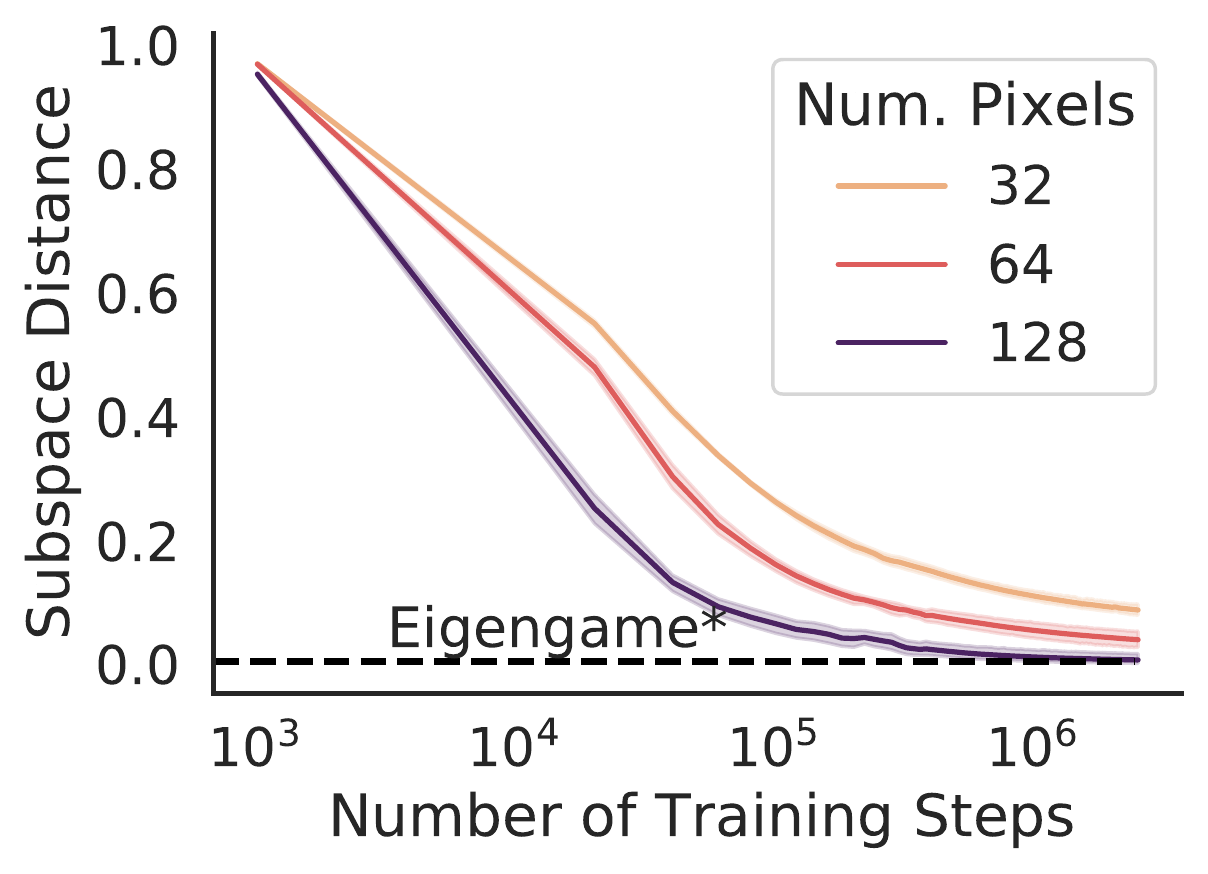}
    \includegraphics[width=0.6\textwidth]{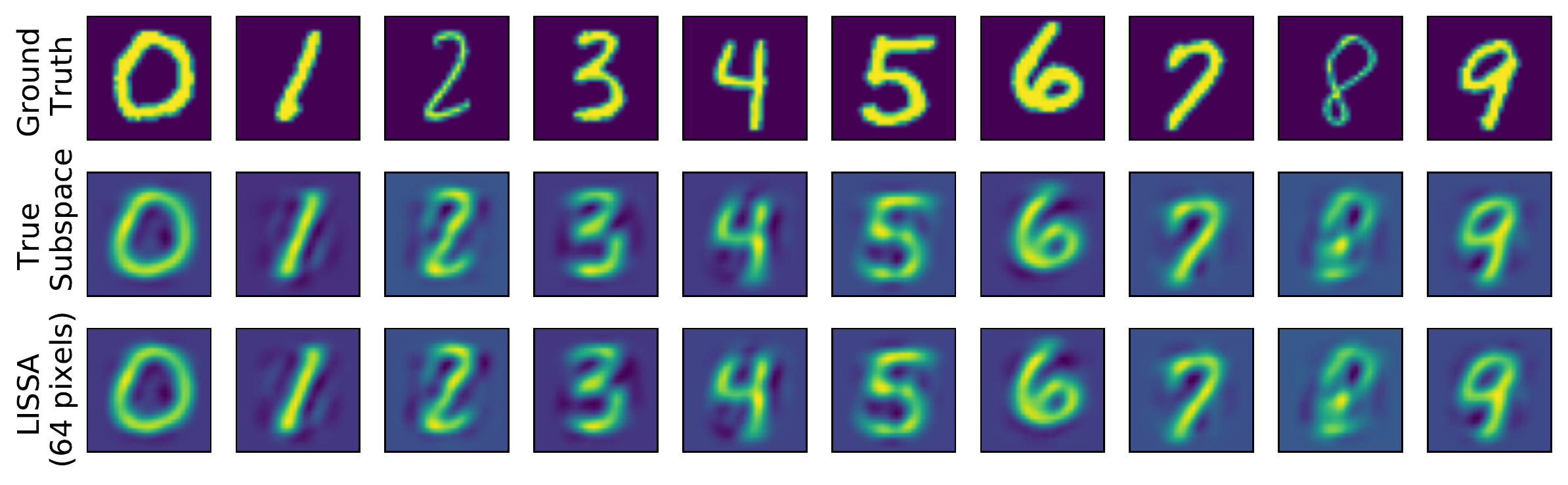}
    \caption{(\textbf{Left}) Training curves for LISSA on MNIST ($d=16$) that updates only a subset of pixels at a time. $*$: see main text. (\textbf{Right}) Reconstruction on MNIST test images. First row show samples from test images. Second are images reconstructed from the true principal components of $\Psi$ and third row are images reconstructed from the principal components learnt by Danskin-LISSA ($N = 64)$. Reconstruction MSE errors for true components and Danskin-LISSA are $21.46$ and $21.53$ respectively.}
    \label{fig:mnist}
    \vspace{-0.15cm}
\end{figure*}

We now conduct an empirical evaluation demonstrating that the Danskin-LISSA algorithm described in  \cref{sec:algo} recovers the $d$-dimensional principal subspace of different types of data: synthetic matrices, MNIST images \citep{lecun2010mnist} and the successor measure for the modified PuddleWorld domain \citep{sutton1995generalization}. In all cases, we measure convergence using the normalized subspace distance \citep{tang2019exponentially} between $\Phi$ and the principal subspace of $\Psi$:
\begin{align*}
    1-\frac{1}{d} \cdot \operatorname{Tr}\left(F_d F_d^\top P_\Phi\right) \in[0,1] .
\end{align*}
Here, $F_d$ are the top-$d$ left singular vectors of $\Psi$ and $P_\Phi = (\Phi^\top \Phi)^\dagger \Phi^\top$ is the orthogonal projector onto the column space of $\Phi$. For simplicity, we take $M = N = J$ in all experiments. The parameter $\kappa = \kappa_0 / \max_{s \in s_{1:J}} \| \phi(s) \|^{2}_2$, where $\kappa_0$ is a hyperparameter, is computed from the sampled feature vectors but we note that it can also be estimated online by a running average.

\subsection{Synthetic Matrices}
\label{sec:synthetic}
To begin, we consider a random matrix $\Psi \in \rR^{50 \times 50}$ whose entries are sampled from a standard normal distribution. 
In order to study our algorithm's behaviour under different conditions, we follow \citet{gemp2020eigengame} and set the matrix's singular values from 1000 to 1 linearly or exponentially (See \cref{app:syntheticmatrices}).
We selected the step size $\alpha=0.001$ and the parameter $\kappa_0=1.9$ from a hyperparameter sweep (\cref{fig:plotkappa} in \cref{app:syntheticmatrices} compares performance for different values of $\kappa_0$, in particular illustrating how $\kappa_0 > 2$ fares poorly, according to our theory).

\cref{fig:plot1a}, left illustrates that the Danskin-LISSA algorithm successfully recovers the $d$-dimensional principal subspace given sufficiently many training steps, with smaller values of $d$ being easier to learn for the exponentially decaying spectrum (results for $d \ge 25$ are given in the appendix). However, we see that learning the subspace spanned by a representation of dimension $d=25$ is easier than $d=1$ for linearly decaying spectrum.
\cref{fig:plot1a} right demonstrates that empirically, it is possible to obtain a reasonable approximation of the principal subspace even for a very smaller number of samples ($J = 1$ being the extreme), despite our theoretical expectation of a biased covariance estimate.

As described in \cref{sec:method}, the Danskin-LISSA approach stems from a combination of several algorithmic concepts. First, it uses two independent estimates of the weight vectors. Second, it embeds a LISSA procedure to estimate the inverse covariance matrix. To understand better their relative importance in the performance of the Danskin-LISSA algorithm, we compare it to two sample-based baselines which have access to the same amount of information and memory. The first one uses the naive gradient estimator described in \cref{sec:method}. The second uses two separate weight estimates, following the derivation from Danskin's theorem, but uses the inverse of the empirical covariance matrix
rather than the LISSA procedure used in the Danskin-LISSA method -- accordingly, we call this the Danskin-Empirical method. \cref{fig:plot2} illustrates the bias-reducing advantage of the LISSA covariance estimator, in particular in the low-sample regime.
The naive method, which constructs a single weight estimate, has high bias and underperforms compared to both of these methods.

\begin{figure*}[t]
    \centering
    \hfill%
   \raisebox{-0.5\height}{\includegraphics[width=0.21\textwidth]{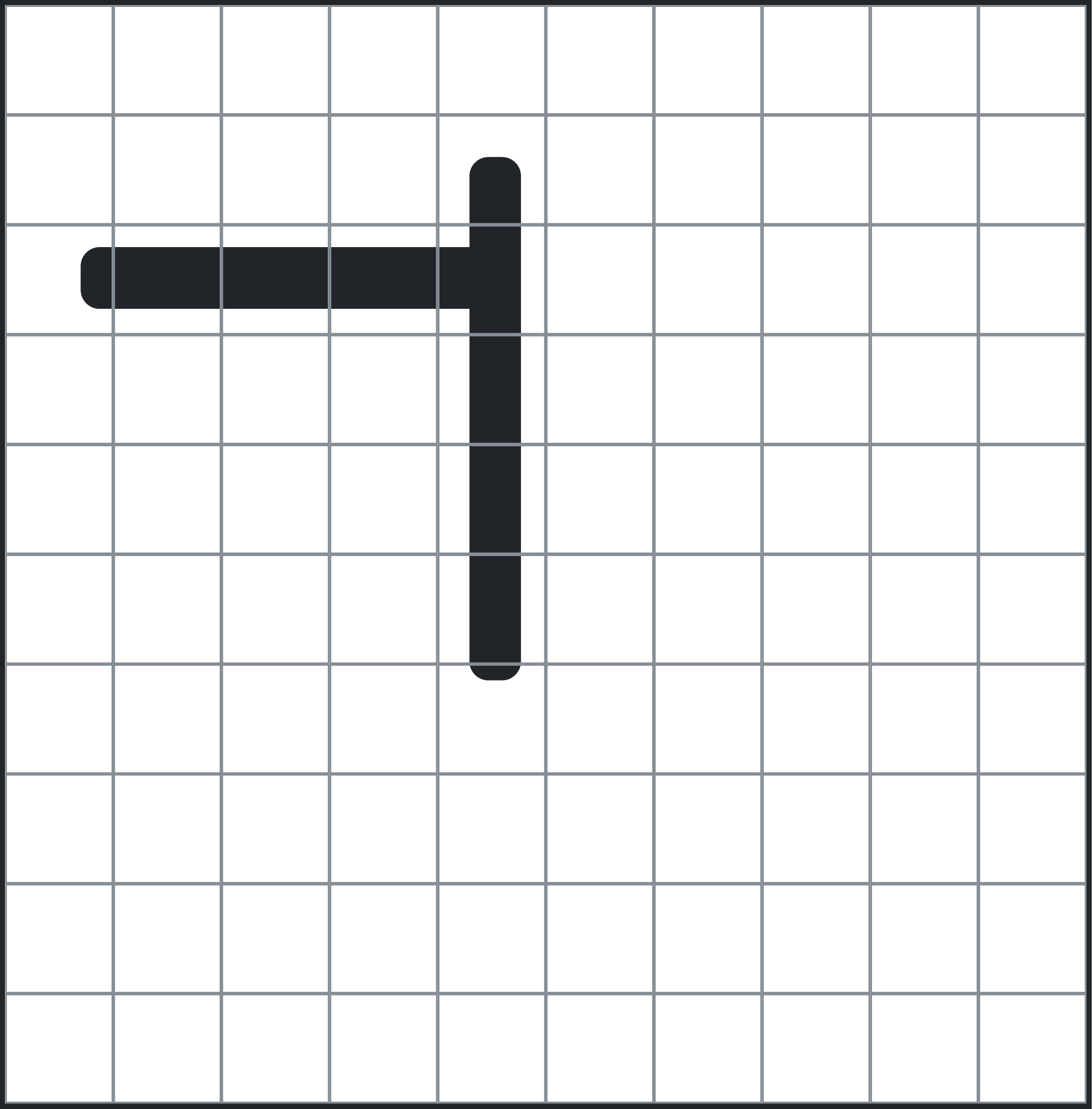}}
    \hfill%
   \raisebox{-0.5\height}{ \includegraphics[width=0.55\textwidth]{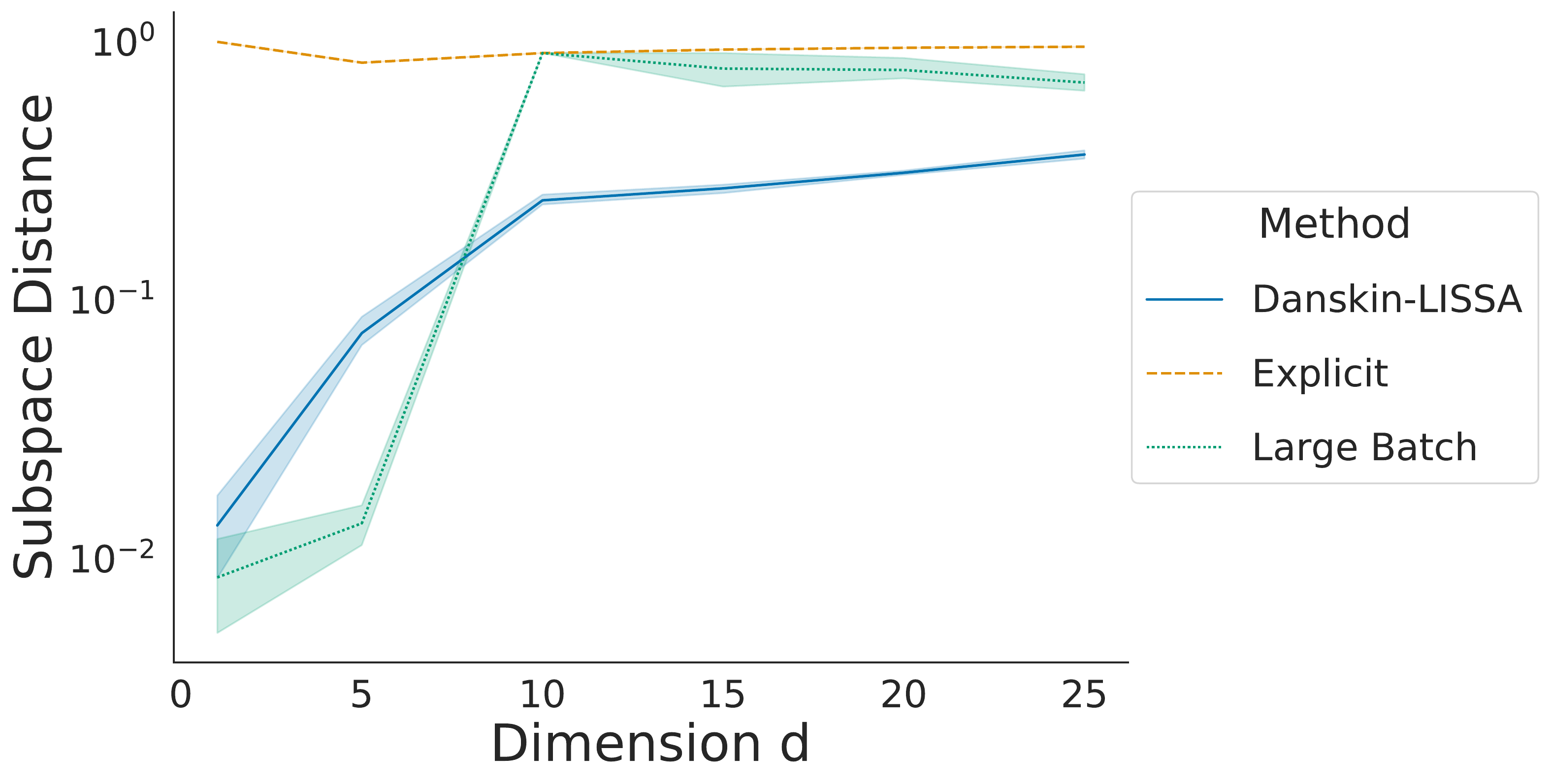}}
    \hfill%
    \hfill%
    \caption{\looseness=-1 (\textbf{Left}) The Puddle World domain \citep{sutton1995generalization}, with the shaded area indicating regions where the agent moves slowly. In our experiments, each grid cell is associated with a column of the implied data matrix. (\textbf{Right}) Subspace distance as a function of the dimension $d$ after $10^8$ gradient steps for three methods: Danskin-LISSA, Explicit, and the Large Batch baseline.}
    \label{fig:pwplot1}
    \vspace{-0.35cm}
\end{figure*}

\subsection{MNIST Dataset}
\label{sec:mnist}
We now consider learning the principal subspace of MNIST images from a training dataset with the Danskin-LISSA algorithm. We represent the data as a matrix $\Psi \in \rR^{784 \times 60 000}$ where each column is a $28 \times 28$ sample image ~(flattened to size $784$) of one of the ten possible digits and from which the mean image has been subtracted.
To accelerate learning speed we use the second-order Adam optimizer \citep{kingma2015adam}.
\cref{fig:mnist} shows that it is possible to effectively learn the principal subspace of this data even while updating as few as 32 pixels (rows) at a time; naturally, using more samples per step results in improved learning speed.
As a point of comparison, we provide the subspace distance obtained by Eigengame \citep{gemp2020eigengame}, a state-of-the-art method that performs PCA by sampling full columns (images) at a time.

To quantify the goodness of the representation learnt on the MNIST training set, we use it to reconstruct MNIST images on the test set.
Denoting $\Psi_{\text{test}} \in \rR^{784 \times 10000}$ the test dataset and $\Phi \in \rR^{784 \times d}$ a representation learnt from the training set, the reconstructed images on the test set are given by $P_\Phi \Psi_{\text{test}}$ where $P_\Phi$ denotes the orthogonal projector onto the column space of $\Phi$.  
\cref{fig:mnist}, right, shows that the MNIST digits reconstructed from the subspace learnt by Danskin-LISSA
qualitatively look similar to the images reconstructed from the true principal components of the training set and achieve a similar reconstruction error.

\subsection{Learning the Successor Measure}
\label{sec:deep-svd}

In reinforcement learning (RL), the successor representation \citep{dayan1993improving} encodes an agent's future trajectories from any given state in terms of the vistation frequency to various states.
Of immediate relevance, it is often used as a building block in representation learning for RL, in particular by directly learning its principal subspace \citep{mahadevan2007proto,behzadian2018low,machado18eigenoption}. \looseness=-1 Its extension to continuous state spaces is called the successor measure \citep{blier2021learning}, and is naturally described by an infinite dimensional matrix. Our last experiment illustrates how the Danskin-LISSA algorithm can be used to approximate the principal subspace of the successor measure of the Puddle World domain \citep{sutton1995generalization}.

In our version of this environment, traversing puddles requires more time, resulting in asymmetric successor measure; details of the environment and the reinforcement learning framework are given in \cref{app:pw}.
Here, $s \in [0, 1]^2$ corresponds to a particular two-dimensional state in the environment.
For a collection of sets $\mathcal{X} = \{ X \subset [0, 1]^2 \}$ to be described below, we define the successor measure as
\begin{equation*}
    \Psi(s, X) = \sum_{t \ge 0} \gamma^t \, \mathbb{P} \left( S_t \in X \, | \, S_0 = s \right), \quad \gamma \in (0, 1)
\end{equation*}
The successor measure describes the expected, discounted number of visits to the set $X$ when the agent begins in state $s$ and moves randomly.
We take $\gamma = 0.99$.

Compared to the experiments of the previous sections, we parametrize the representation by a neural network.
We are interested in understanding the degree to which this neural network can be trained to approximate the $d$-dimensional principal subspace of the successor measure. We take the collection $X$ to be the set of non-overlapping cells of a $100 \times 100$ grid (illustrated by \cref{fig:pwplot1}). For computational reasons, we assign the same value of $\Psi(\cdot, X)$ to all states within a grid cell; this value is computed by 
$1,000$ truncated Monte-Carlo rollouts from a start state sampled uniformly at random within a cell.
This produces a $10,000 \times 10,000$ matrix which we treat as ground truth for measuring the accuracy of our predicted subspace.

To gain an understanding of the effectiveness or our method, we compare it with two other gradient-based methods commonly used in reinforcement learning. As the name indicates, the Explicit method maintains a weight vector $w_i$ for each column and relies on the pair of updates from Equation \ref{eqn:explicit_method}, similar to the method used by \citet{bellemare2019geometric,lyle2021effect}.
Note that we present this method only for completeness, as it is not applicable to an infinite number of columns and may otherwise carry an impractically large memory cost.
The Large Batch method, on the other hand, estimates the weight vector $w_i$ using $\phi$ and $\Psi$ evaluated at center of each of the 10,000 grid cells (close in spirit to the Naive method of Section \ref{sec:synthetic}). %

All three methods use Adam \citep{kingma2015adam} to optimize a two-layer
MLP with $512$ hidden units and ReLU activations. We take $J = M = N = 50$ for Danskin-LISSA and $N = 250$ for the two other methods. The step size $\alpha$ was tuned for each method according to a small hyperparameter sweep and after $10^8$ gradient steps averaged across 5 runs.
Details outlining these sweeps and complete experimental methodology can be found in \cref{app:expdetails}.

\cref{fig:pwplot1}, right compares the final subspace distance of these three algorithms
for various values of $d$. We find that the performance of the Danskin-LISSA algorithm degrades gracefully as $d$ is increased, while the Large Batch method is only practical for small values of $d$.
In part, this is explained by the fact that even with such a large batch, there is a residual bias in the latter method's covariance estimate. 
The poor performance of the Explicit method is explained by the fact that a single column is updated at any given time, resulting in stale weight vectors $w_i$. Although in practice this can be mitigated by updating multiple columns at once, the result illustrates an important pitfall with the use of an explicit weight vector.

\section{DISCUSSION \& CONCLUSION}
\label{sec:conc}
In this paper, we presented an algorithm that learns principal components of very large or infinite dimensional matrices by stochastic gradient descent.
Our experiments on synthetic matrices and MNIST images demonstrate that indeed the method converges towards their top principal subspace. 
Our analysis on the Puddle World domain also demonstrates that our algorithm can learn a low-dimensional, neural-network state representation. In deep reinforcement learning (RL), training a network on supervised auxiliary predictions results in its representation corresponding to the principal components of this set of tasks, assuming the network is other unconstrained \citep{bellemare2019geometric}.  Incorporating the Danskin-LISSA procedure within a deep RL architecture may provide performance improvements by incorporating more knowledge about the world into the network’s representation.

For simplicity, in this paper we assumed that all samples used in computing a given gradient estimate are drawn independently. In practice, samples are naturally expensive and it may appear undesirable to require a total of $N + 2J + 2M$ for a single gradient estimate. However, one can improve on this state of affairs by permuting the order in which samples from the batch are presented, constructing different gradient estimates from these permutations, and noting that the average of multiple unbiased estimates remains unbiased (and generally has lower variance).

\section*{Acknowledgements}
The authors would like to thank Ian Gemp, Matthieu Geist  and the anonymous reviewers for useful discussions and feedback on this paper.

We would also like to thank
the Python community \citep{van1995python,oliphant2007python} for developing
tools that enabled this work, including
{NumPy}~\citep{oliphant2006guide,walt2011numpy, harris2020array},
{SciPy}~\citep{jones2001scipy},
{Matplotlib}~\citep{hunter2007matplotlib} and JAX \citep{bradbury2018jax}.

\bibliographystyle{plainnat}
\bibliography{bib} 
\clearpage
\onecolumn
\begin{appendix}
\section{Proofs for \cref{sec:background}}
\label{app:proofs}

\begin{restatable}[]{lemma}{}
\label{lemma:w_star}
Following the notations from \cref{sec:background}, we have
\begin{equation}
    W^*_\Phi = (\Phi^\top \Xi \Phi)^\dagger \Phi^\top \Xi \Psi
\end{equation}
\end{restatable}
\begin{proof}
For a fixed $\Phi \in \rR^{S \times d}$,
\begin{align*}
    \nabla_{W} \cL(\Phi, W) &= \nabla_{W} \|\Xi^{1/2}(\Phi W -\Psi) \Lambda^{1/2} \|^2_F\\
    &= 2 \Phi^\top \Xi (\Phi W - \Psi) \Lambda
\end{align*}
\begin{align*}
     \nabla_{W} \cL(\Phi, W) =0    &\Longleftrightarrow  2 \Phi^\top \Xi (\Phi W^*_\Phi - \Psi) \Lambda = 0\\
   &\Longleftrightarrow  \Phi^\top \Xi (\Phi W^*_\Phi - \Psi) = 0  \text{ as $\lambda(t)>0$ for all $t \in \cT$}\\
   &\Longleftrightarrow \Phi^\top \Xi \Phi W^*_\Phi = \Phi^\top \Xi \Psi \\
      &\Longleftrightarrow  W^*_\Phi = (\Phi^\top \Xi \Phi)^\dagger \Phi^\top \Xi \Psi \\
\end{align*}
\end{proof}

\begin{restatable}[]{proposition}{svd}
\label{prop:svd}
Let $\textrm{GL}_d(\rR)$ be the set of $d \times d $ invertible matrices. Assume $\Psi$ has strictly decreasing singular values and $\textrm{rank}(\Psi)=r < \infty$ Write $\Psi = F \Sigma B^\T$ for the SVD of $\Psi$ with respect to the inner products $\langle x, y\rangle_{\Xi}$ for all $x, y \in \rR^{S}$ and $\langle x, y \rangle_\Lambda$ for all $x, y \in \rR^T$. For an integer $\ell \in \{1, ..., S\}$, let
$F_\ell \in \rR^{S \times \ell}$ be the matrix
containing the first $\ell$ columns of $F$ (sorted by decreasing singular value).
For a fixed $d \in  \{1, ..., r\}$, 
\begin{align}
\argmin_{\Phi \in \rR^{S \times d}} \min_{W \in \rR^{d \times T}}  \|\Xi^{\frac{1}{2}}(\Phi W - \Psi) \Lambda^{\frac{1}{2}} \|^2_F = \{ \Phi \in \rR^{S \times d} \mid \exists M \in GL_d(\rR), \Phi = F_d M \}  \, .
\label{eq:svd}
\end{align}
\end{restatable}
\begin{proof}
We have $\Psi= F \Sigma B^\T$ where $F \in \rR^{S \times r}$, $\Sigma \in \rR^{r \times r}$ and $B \in \rR^{T \times r}$ satisfy $F^\T \Xi F=I, B^\T \Lambda B = I$. Let $F_d, \Sigma_d$ and $B_d$ the matrices containing the first $d$ columns of $F, \Sigma$ and $B$ respectively.
For a fixed $\Phi \in \rR^{S \times d}$ and if $\Phi$ is full rank, the unique solution of $\min_{W \in \rR^{d \times T}} \||\Xi^{\frac{1}{2}}(\Phi W - \Psi) \Lambda^{\frac{1}{2}} \|^2_F$ is given by $W^*_\Phi =(\Phi^\T \Xi \Phi)^{-1}\Phi^T \Xi \Psi$. When $\Phi$ is orthonormal with respect to the inner product induced by $\Xi$, we have $\Phi^\T \Xi \Phi = I$ and $W^*_\Phi = \Phi^\T \Xi \Psi$.
Moreover, $\rank(\Phi W^*_\Phi) \leq \min(\rank(\Phi), \rank(\Phi^\T \Xi \Psi)) \leq \min(d, \min(d, S, r)) = d$. By the Eckart-Young theorem, given a target matrix $\Psi$, the best approximating matrix of rank at most $d$, with respect to the norm induced by $\Xi$, is $F_d \Sigma_d B_d^\T$ which can be written in terms of an orthogonal projection as follows $F_d F_d^\T \Xi \Psi$. By identification, $\Phi W^*_\Phi = \Phi (\Phi^T \Xi \Psi) = F_d (F_d^\T \Xi \Psi)$ and $\Phi=F_d$ is a solution to \cref{eq:svd}.

As we can turn the basis $\Phi$ for $\spn(F_d)$ into any other basis $\Phi' = \Phi R$ with $R \in \rR^{d \times d}$ an invertible matrix, the set of solutions for $\Phi$ is $\{F_d R : R \in \rR^{d \times d} \text{ invertible} \}$
\end{proof}

\section{Proofs for \cref{sec:algo}}
\label{app:proofsalgo}
Let $\Xi = \E_{s \sim \nu}[e_s e_s^\T]$ and $\Lambda = \E_{t \sim \Lambda}[e_t e_t^\T]$.

\begin{restatable}[]{lemma}{neumanseries}
The $j$-LISSA estimator $\widehat{\Delta}_{j}$ is an unbiased estimator of the partial Neumann series defined in \cref{eqn:neumann_series_for_pseudoinverse}. That is, given $j$ samples $s_{1:j}=\{s_1, s_2, ..., s_j \}$ drawn i.i.d. from $\xi$, we have that
\begin{align*}
\expect_{s_{1:j}\sim \xi}[{\widehat{\Delta}_j}] = \kappa \sum_{i=0}^j (I -  \kappa \Phi^\top\Xi \Phi)^i
\end{align*}
\label{lemma:neumanseries}
\end{restatable}
\begin{proof}
By induction.
\begin{align*}
&\E[\widehat{\Delta}_0] = \E[ \kappa I] = \kappa I \text{ and } \kappa \sum_{i=0}^0 (I - \kappa \Phi^\T \Xi \Phi)^i = \kappa I \\
&\E_{s_1 \sim \xi}[\widehat{\Delta}_1] = \E_{s_1 \sim \xi}[ \kappa I + (I - \kappa \phi_{s_1} \phi_{s_1}^\T)\kappa I] = \kappa I + \kappa(I - \kappa \Phi^\T \Xi \Phi) \text{ as } \E[\phi_i \phi_i^\T] = \Phi^\T \Xi \Phi\\
&\text{and }\kappa \sum_{i=0}^1 (I - \kappa \Phi^\T \Xi \Phi)^i = \kappa I + \kappa(I - \kappa \Phi^\T \Xi \Phi) 
\end{align*}
Let's suppose that $\E_{s_{1:j-1} \sim \xi}[{\widehat{\Delta}_{j-1}}] = \kappa \sum_{i=0}^{j-1} (I - \kappa \Phi^\T N \Phi)^i.$ Then,
\begin{align*}
    \E_{s_{1:j} \sim \xi}[\widehat{\Delta}_{j}] &= \E_{s_{1:j}}[\kappa I + (I - \kappa \phi_{s_j} \phi_{s_j}^\T)\widehat{\Delta}_{j-1}]\\
    &= \kappa I + \E_{s_{1:j}}[(I - \kappa \phi_{s_j} \phi_{s_j}^\T)\widehat{\Delta}_{j-1}]\\
    &= \kappa I + \E_{s_j\sim \xi}[I - \kappa \phi_{s_j} \phi_{s_j}^\T] \E_{s_{1:j-1} \sim \nu} [\widehat{\Delta}_{j-1}] \\
    &= \kappa I + (I - \kappa \Phi^\T N \Phi) \kappa \sum_{i=0}^{j-1} (I - \kappa \Phi^\T N \Phi)^i\\
    &= \kappa \sum_{i=0}^j (I - \kappa \Phi^\T N \Phi)^i
\end{align*}
Hence, the conclusion.
\end{proof}

\bias*
\begin{proof}
\begin{align*}
 \text{bias}(\widehat{\Delta}_{j}) &= \E(\widehat{\Delta}_{j}) -  (\Phi^\T \Xi \Phi)^{\dag}\\
&= \kappa \sum_{i=0}^j (I - \kappa \Phi^\T \Xi \Phi)^i - (\Phi^\T \Xi \Phi)^{\dag} \text{ by \cref{lemma:neumanseries}}\\
 &= \kappa (I- (I - \kappa \Phi^\T \Xi \Phi))^{\dag} (I - (I - \kappa \Phi^\T \Xi \Phi)^{j+1}) - (\Phi^\T \Xi \Phi)^{\dag} \text{ using the closed form of a geometric series}\\
  &= -( \Phi^\T \Xi \Phi)^{\dag} (I - \kappa \Phi^\T \Xi \Phi)^{j+1}
\end{align*}
\end{proof}

\unbiasedestimator*
\begin{proof}
By definition,
\begin{align*}
\hat g_{\textsc{dl}}(s) = \hat w_i' \big(\phi(s)^\top \hat w_i - \psi_i(s)\big)
\end{align*}
Plugging in $ \hat w_i = \hat C \phi(s') \psi_i(s')$ and  $\hat w_i' = \hat C' \phi(s'') \psi_i(s'')$, we have
\begin{align*}
   \hat g_{\textsc{dl}}(s)^\top = \big(\phi(s)^\top  \hat C \phi(s') \psi_i(s') - \psi_i(s)\big) (\hat C' \phi(s'') \psi_i(s''))^\top    
\end{align*}
Now taking the expectation,
\begin{align*}
  \E_{s, s', s'', s_{1:n}, s^\prime_{1:n}, i} [e_s \hat g_{\textsc{dl}}(s)^\top] &=   \E_{s, s', s'', s_{1:n}, s^\prime_{1:n}, i} \Big[e_s \big(\phi(s)^\top  \hat C \phi(s') \psi_i(s') - \psi_i(s)\big) (\hat C' \phi(s'') \psi_i(s''))^\top \Big]  \\
  &= \E_{s, i}  \Big[e_s \big(\phi(s)^\top  \E_{s_{1:n}}[\hat C] \E_{s'}[ \phi(s') \psi_i(s')] - \psi_i(s)\big) (\E_{s^\prime_{1:n}}[\hat C'] \E_{s''}[\phi(s'') \psi_i(s'')])^\top \Big]  \\
    &= \E_{s, i}  \Big[e_s \big(e_s^\top \Phi  \E_{s_{1:n}}[\hat C] \E_{s'}[ \Phi^\T e_{s'} e_{s'}^\T \Psi e_i] - e_s^\top \Psi e_i\big) (\E_{s^\prime_{1:n}}[\hat C'] \E_{s''}[\Phi^\T e_{s''} e_{s''}^\T \Psi e_i])^\top \Big]   \\
    &= \E_{s, i}  \Big[e_s e_s^\top \big(\Phi  \E_{s_{1:n}}[\hat C] \E_{s'}[ \Phi^\T e_{s'} e_{s'}^\T \Psi] - \Psi \big) e_i e_i^\T  \E_{s''}[\Psi^\T e_{s''} e_{s''}^\T \Phi] (\E_{s_{1:n}^\prime}[\hat C'])^\top  \Big]  \\
    &= \E_s e_s e_s^\top \big(\Phi  \E_{s_{1:n}}[\hat C] \Phi^\T \E_{s'}[e_{s'} e_{s'}^\T] \Psi - \Psi\big)\E_{i}[ e_i  e_i^\T ] \Psi^\T \E_{s''}[ e_{s''} e_{s''}^\T] \Phi \E_{s_{1:n}^\prime}[\hat C']   \\
&= \Xi \big(\Phi  \E_{s_{1:n}}[\hat C] \Phi^\T \Xi \Psi - \Psi\big)\Lambda (\Psi^\T \Xi\Phi (\E_{s_{1:n}^\prime}[\hat C'])^\top )  \\
\end{align*}
where in the last line, we used the fact that $\Xi = \E_{s \sim \nu}[e_s e_s^\T]$ and $\Lambda = \E_{t \sim \Lambda}[e_t e_t^\T]$.
Now, given two unbiased estimators $\hat{C}$ and $\hat{C'}$, we have
\begin{align*}
    \E_{s_{1:n}}[\hat{C}] = (\Phi \Xi \Phi^\top)^\dagger \text{ and }  \E_{s_{1:n}^\prime}[\hat{C'}] = (\Phi \Xi \Phi^\top)^\dagger 
\end{align*}
It then follows that
\begin{align*}
      \E_{s, s', s'', s_{1:n}, s_{1:n}^\prime, i}[e_s \hat g^\top_{\textsc{dl}}(s)] &= \Xi \big(\Phi  (\Phi \Xi \Phi^\top)^\dagger  \Phi^\T \Xi \Psi - \Psi \big)\Lambda (\Psi^\T \Xi)\Phi (\Phi \Xi \Phi^\top)^\dagger\\
      &= \Xi \big(\Phi  (\Phi \Xi \Phi^\top)^\dagger  \Phi^\T \Xi \Psi - \Psi \big)\Lambda ((\Phi \Xi \Phi^\top)^\dagger \Phi^\T \Xi \Psi)^\T\\
      &=  \Xi (\Phi W^*_\Phi - \Psi) \Lambda (W^*_\Phi)^\top\\
      &\propto  \nabla_{\Phi} \cL(\Phi) 
\end{align*}
\end{proof}

\section{Additional Experimental Results}
\label{app:expdetails}
\subsection{Synthetic matrices}
\label{app:syntheticmatrices}
\begin{figure*}[t]
    \centering
 \includegraphics[width=0.49\textwidth]{images/plot_1a_5.pdf}
  \includegraphics[width=0.49\textwidth]{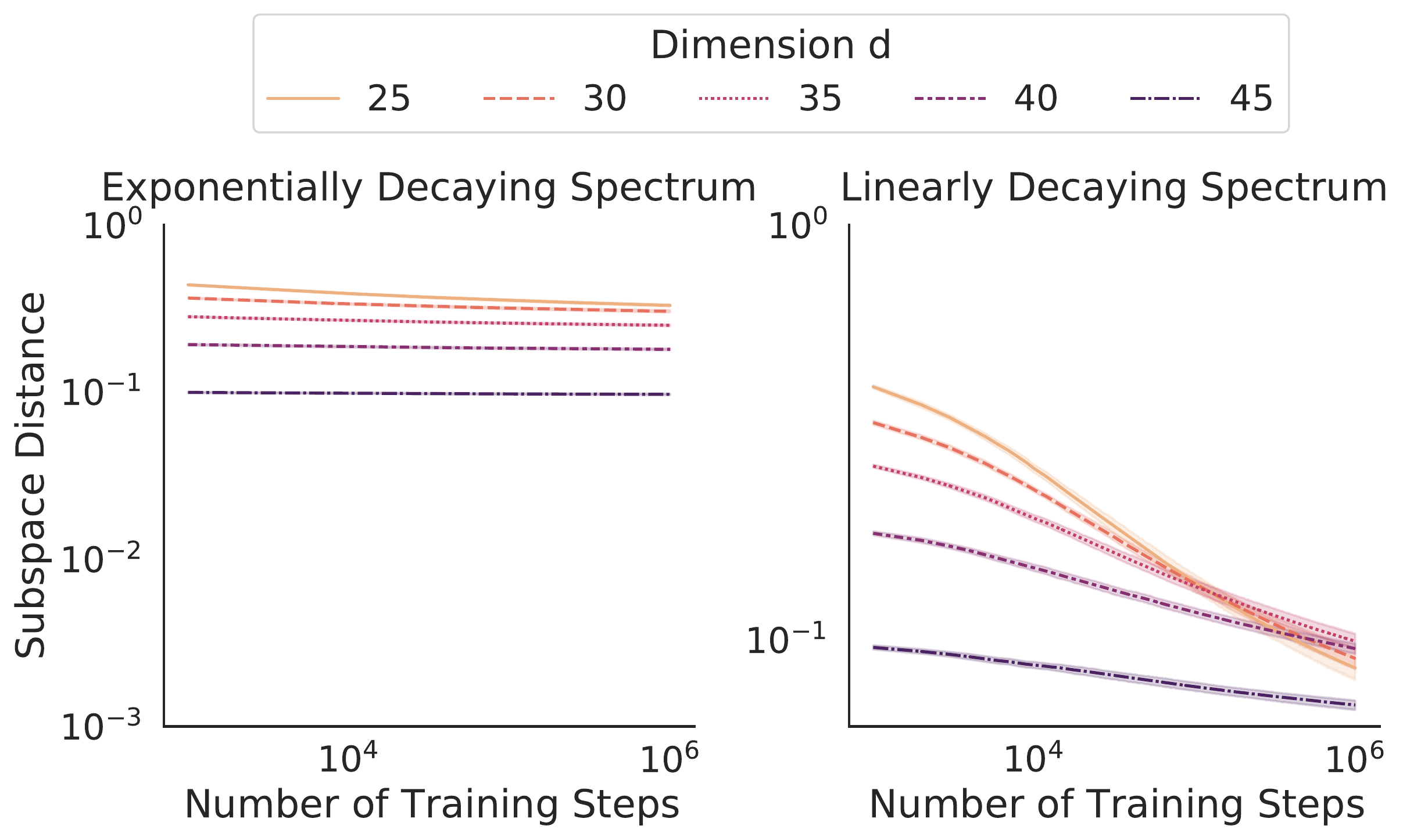}
    \caption{\looseness=-1 Subspace distance over the course of training LISSA for different dimensions on synthetic matrices with a spectrum decaying linearly and exponentially, averaged over 30 seeds. The total number of samples used is 50. Shaded areas represent 95$\%$ confidence intervals.}
    \label{fig:plot1aapp}
\end{figure*}

\begin{figure}[h]
    \centering
    \includegraphics[width=0.8\textwidth]{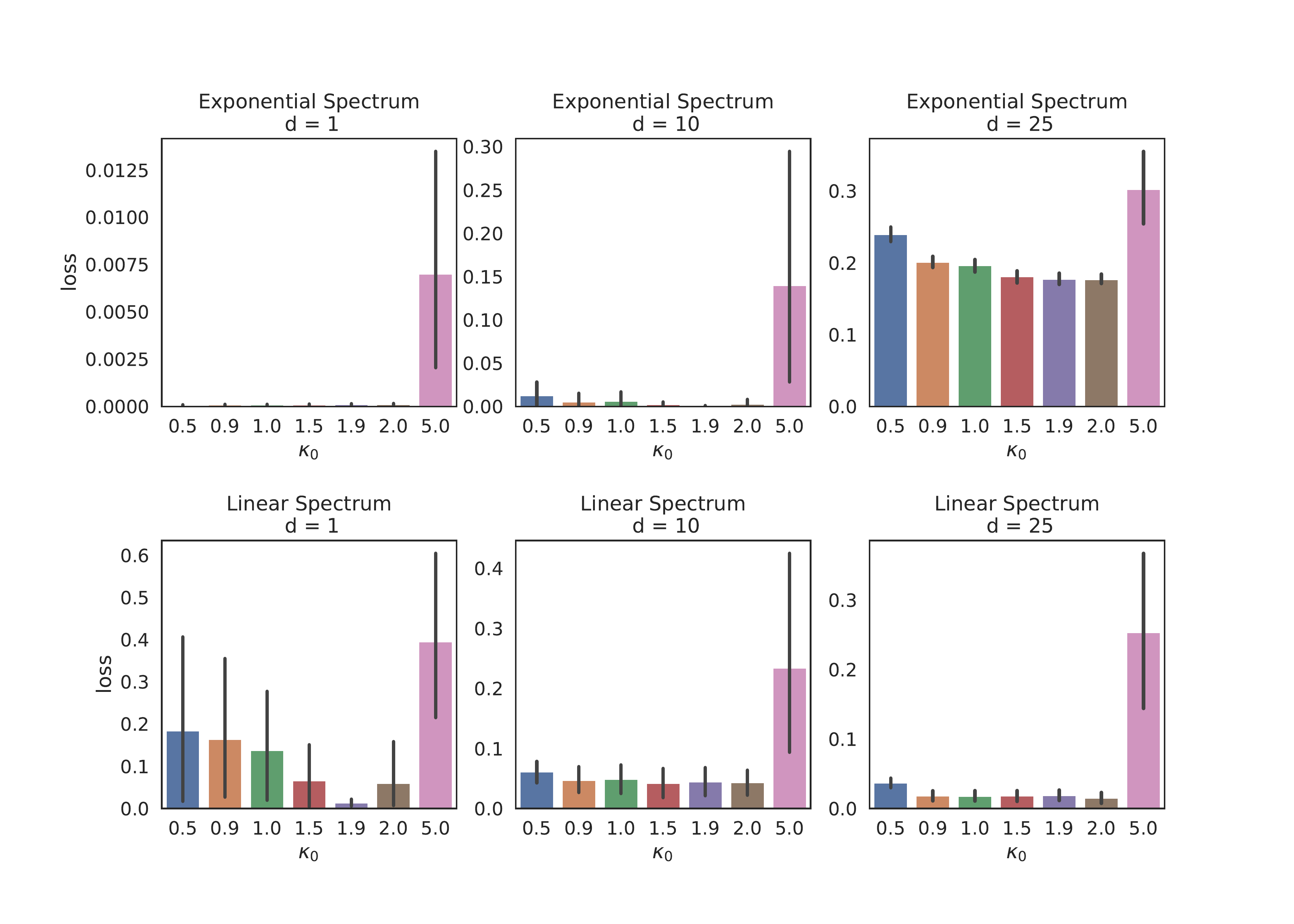}
    \vspace{-0.3cm}
    \caption{\looseness=-1 Subspace distance after $10^6$ training steps of the LISSA algorithm for different $\kappa_0$}
    \label{fig:plotkappa}
\end{figure}
We follow the experimental protocol from \citet{gemp2020eigengame}. We initialize $\Psi \in \rR^{50 \times 50}$ randomly from a normal distribution. We compute its SVD such that $\Psi = F \Sigma B$. Let $\Sigma_{\text{linear}}= \mathrm{diag}(1, ..., 1000)$ and  $\Sigma_{\text{exp}}= \mathrm{diag}(10^0, ..., 10^3)$. We rescale the matrix $\Psi$ such that $\Psi_{\text{linear}}=F\Sigma_{\text{linear}}B$ and $\Psi_{\text{exp}}=F \Sigma_{\text{exp}} B$. The matrix $\Phi \in \rR^{S \times d}$ is also initialized randomly from a standard normal distribution.
We sweeped over the step size $\alpha$ and chose $\alpha=0.001$ which was working well in all the synthetic experiments. We used the SGD otpimizer but found that there was not a big performance difference with the Adam otpimizer \citep{kingma2015adam} in most of these synthetic experiments. In \cref{fig:plotkappa}, we also sweeped over the hyperparameter $\kappa_0$ and found that $\kappa_0=1.9$ was performing well across dimensions and for both linear and exponential spectra. 
We trained the Danskin-LISSA method for $10^6$ time steps.
As a complement to \cref{fig:plot1a}, we show in \cref{fig:plot1aapp} the training curves of the Danskin-LISSA algorithm for a broader range of dimensions $d$. For the exponential spectrum, when $d \geq 25$, larger dimensions are easier to learn. This is the opposite trend to the behavior found when $d \leq 25$ where smaller dimensions are easier to learn. For the linearly decaying spectrum, when $d \geq 25$, larger dimensions are easier to learn which is also the same trend as what we observed for $d \leq 25$.

\subsection{MNIST}
\label{app:mnist}
We found that the Adam optimizer \citep{kingma2015adam} performed best for our MNIST experiments. We performed a sweep over the step-size $\alpha$ and found that $\alpha = 0.005$ worked best for $128$ and $64$ pixels. $\alpha=0.01$ performed best for $32$ pixels. We 
trained the Danskin-LISSA algorithm for $2.5 \times 10^6$ steps. 

\subsection{Puddle World}
\label{app:pw}
\begin{figure}[t]
    \centering
    \includegraphics[width=0.9\textwidth]{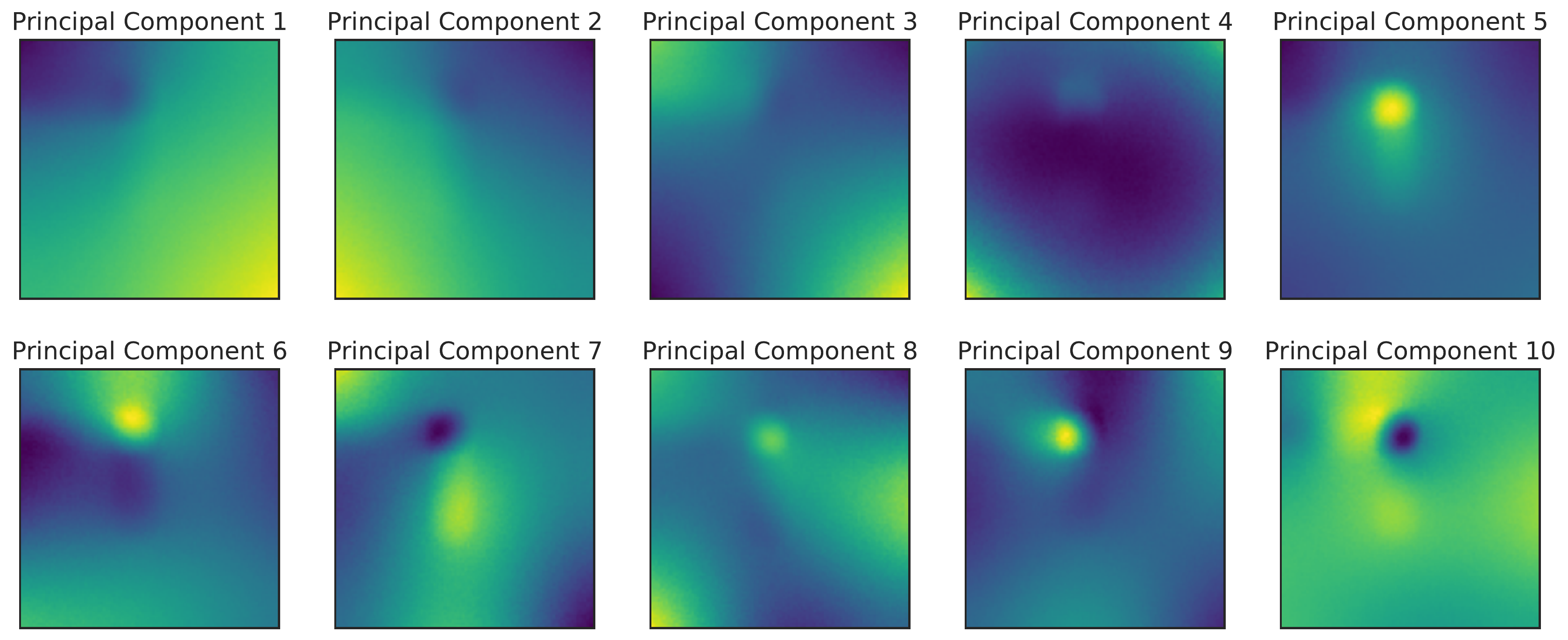}
    \vspace{-0.3cm}
    \caption{\looseness=-1 First 10 principal components of the successor measure of the Puddle World domain.}
    \label{fig:pwpc}
\end{figure}
A Puddle World \citep{sutton1995generalization} is a square arena, with x, y both in [0, 1]. It has a continuous state space and a discrete action space.
There are four actions (up, down, left, right) that move the agent by 0.05 in each of the corresponding directions. A random gaussian noise with standard deviation $0.01$ is also added to transitions in both directions. For our experiments, we used the same puddle configuration found in \citep{sutton1995generalization}. This configuration contains two puddles. The first puddle lies between the points $(0.1, 0.75)$ and $(0.45, 0.75)$ with a radius of 0.1. The second puddle lies between the points $(0.45, 0.4)$ and $(0.45, 0.8)$, also with a radius of $0.1$. While the original Puddle World gives negative rewards for being in a puddle, our puddles instead cause a slowing affect by a factor of $0.5$. That is, when in a puddle, the agent only moves by $0.025$ in each direction. The puddles compound, meaning that in the area where the two puddles overlap the agent will only move a distance of $0.0125$. We chose to use slowing puddles because our task is reward-agnostic, and the successor measure task that we chose would capture the dynamics of the slowing puddles. We visualize in \cref{fig:pwpc} the top-$10$ principal components of the successor measure of Puddle World, demonstrating that they are non-trivial.

The successor measure was computed using $1000$ Monte Carlo rollouts from each starting grid cell, truncated after $700$ steps. We used a discount factor $\gamma = 0.99$.
We subtracted the row sums to center-mean each column of the ground truth matrix $\Psi \in \rR^{10^4 \times 10^4}$. 

For each of the methods, we performed a sweep of learning rates and optimizers (between Adam and SGD) and found that Adam with a learning rate of $10^{-4}$ worked well across the board. We ran each method for 100 million gradient steps. For Danskin-LISSA, we kept $\kappa$ fixed at $1.9$, which we found worked well in our previous experiments. Danskin-LISSA used a batch size of $50$ for each of its $5$ batches, while Large Batch and Explicit used a main batch size of $250$ to ensure that each method saw the same number of samples. To compute $\phi(s)$ we used a two hidden-layer MLP with 512 hidden units per layer.

\end{appendix}

\end{document}